\documentclass{article}

\PassOptionsToPackage{numbers, compress}{natbib}
\usepackage{natbib}

\usepackage[dblblindworkshop, final]{neurips_2025}
\usepackage[utf8]{inputenc}
\usepackage[T1]{fontenc}
\usepackage{graphicx}
\usepackage{subcaption}
\usepackage{comment}
\usepackage{tabularx}
\usepackage{listings}
\usepackage{xcolor}
\usepackage{amsmath}
\usepackage{tikz}
\usepackage{pgfplots}
\usetikzlibrary{positioning}
\usetikzlibrary{arrows.meta}
\usepackage{algorithm}
\usepackage{algpseudocode}
\usepackage{pgfplots}
\usepgfplotslibrary{groupplots} 
\pgfplotsset{compat=1.18}       
\usepackage{amsthm}
\newtheorem{proposition}{Proposition}
\usepackage{array} 

\workshoptitle{What Makes a Good Video: Next Practices in Video Generation and Evaluation}

\usepackage[utf8]{inputenc} 
\usepackage[T1]{fontenc}    
\usepackage{hyperref}       
\usepackage{url}            
\usepackage{booktabs}       
\usepackage{amsfonts}       
\usepackage{nicefrac}       
\usepackage{microtype}      
\usepackage{xcolor}         
\usepackage{cleveref}
\usepackage{wrapfig}
\usepackage{pgfplotstable}

\title{The Unwinnable Arms Race of AI Image Detection}

\author{%
  Till Aczel \quad Lorenzo Vettor \quad Andreas Plesner \quad Roger Wattenhofer \\
  ETH Zürich, Switzerland \\
  \texttt{\{taczel, lvettor, aplesner, wattenhofer\}@ethz.ch}
}

\begin{document}

\maketitle

\begin{abstract}
The rapid progress of image generative AI has blurred the boundary between synthetic and real images, fueling an arms race between generators and discriminators. This paper investigates the conditions under which discriminators are most disadvantaged in this competition. We analyze two key factors: data dimensionality and data complexity. While increased dimensionality often strengthens the discriminator’s ability to detect subtle inconsistencies, complexity introduces a more nuanced effect. Using Kolmogorov complexity as a measure of intrinsic dataset structure, we show that both very simple and highly complex datasets reduce the detectability of synthetic images; generators can learn simple datasets almost perfectly, whereas extreme diversity masks imperfections. In contrast, intermediate-complexity datasets create the most favorable conditions for detection, as generators fail to fully capture the distribution and their errors remain visible.
\end{abstract}

\section{Introduction}

Over the past decade, generative AI has enabled highly realistic synthetic media, including deepfakes \citep{westerlund2019emergence}. 
These technologies blur the line between reality and fabrication, creating significant societal challenges \citep{de2021distinct}. 
While these advances have opened new possibilities in art and design \citep{zhou2024art}, they have also introduced risks in disinformation, fraud, and media authenticity verification \citep{9049288,verma2025deepfakes}. 
Reports show thousands of deepfake attacks annually, causing hundreds of millions in financial losses and eroding public trust in digital media \citep{venturebeatdeepfake2025, drjdeepfake2025, surfsharkdeepfake2025, wsjdeepfake2025, sippy2024behind}. 
Despite growing awareness, unaided human observers perform only slightly better than chance at distinguishing AI-generated images from real photographs \citep{roca2025good, groh2021}.
Traditional verification systems struggle to detect AI-generated content highlighting the urgent need for robust detection methods \citep{mahara2025methods}.
The ability to distinguish synthetic images from real ones has therefore become increasingly important, both for security and for maintaining trust in digital media.  

This dynamic has evolved into an arms race between generators, which strive to produce indistinguishable samples, and discriminators, which attempt to detect fakes.  
Over time, both AI-generated content and detection methods will improve, but the battle remains inherently asymmetric: if a generator perfectly captures the data distribution, no discriminator can ever win. 
Thus, the generator can always improve and approach a point where detection becomes impossible.
Understanding the limits of detection is crucial for developing reliable tools to safeguard digital content.  

Existing benchmarks focus on selecting the best discriminator.  
Little is known about the conditions under which discriminators are most disadvantaged, particularly when considering the full spectrum from simple, structured datasets to highly diverse, complex ones.  
For example, simple datasets include MNIST, which consists of centered grayscale digits with minimal variation, whereas complex datasets include CIFAR-10, which contains small color images across ten classes with significant variability in objects, backgrounds, and lighting.
We quantify complexity in terms of Kolmogorov complexity, which measures the inherent compressibility or structure of a dataset.  
This metric is particularly relevant in the context of generative modeling and detection, as datasets that are highly compressible are easier for generators to reproduce and harder for discriminators to exploit, whereas less compressible datasets introduce variability that challenges both sides.
As dataset complexity increases, the task becomes more challenging for both the generator and the discriminator.  
When the task is very simple, the generator can achieve near-perfect modeling, leaving the discriminator at a disadvantage.  
Conversely, if the dataset is extremely complex, neither the generator nor the discriminator can fully capture the distribution, and the discriminator again struggles to reliably detect fakes.  
In this sense, intermediate complexity presents a unique regime where the generator is imperfect but the data is structured enough for the discriminator to identify inconsistencies.

By examining this spectrum from simple to complex datasets and low to high dimensionality, we aim to map the conditions under which synthetic data are easiest and hardest to detect.  
Our contributions are as follows:
\begin{enumerate}
    \item \textbf{Formal proof of the inherent challenge of detection:} We show that distinguishing generated image content from real images is an unwinnable battle, establishing theoretical limits for discriminators.
    \item \textbf{Impact of dataset complexity:} We systematically analyze how the complexity of datasets, measured through Kolmogorov complexity, affects the detectability of synthetic images.
    \item \textbf{Role of input resolution:} We quantify how changes in image resolution influences the ability of discriminators to detect synthetic images.
\end{enumerate}

Our work serves both as a theoretical study, establishing formal limits on detectability, and as a conceptual framework, mapping how dataset complexity and resolution shape discriminator performance.
Together, these perspectives reveal both the long term impossibility of perfect detection and the practical regimes where it remains feasible.

\begin{figure}[t]
    \centering
    \includegraphics[width=\textwidth]{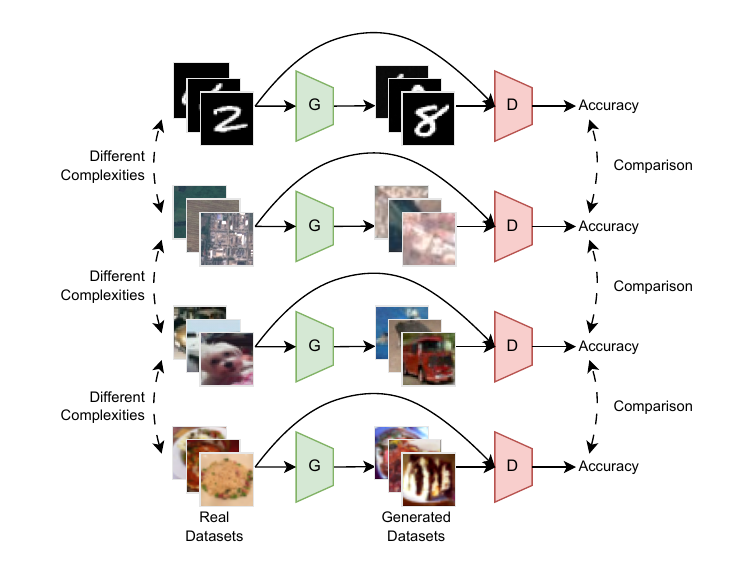}
    \caption{overview of the experimentation setup used in this project. We take multiple datasets with different (Kolmogorov) complexities or resolutions, and independently train copies of the same diffusion-based image generator and convolutional discriminator for the datasets. We then compare the accuracies of the discriminators against differences in the dataset complexities.}
    \label{fig: overview}
\end{figure}

\section{Related Works}

Over the past decade, image generation has advanced rapidly, evolving from Variational Autoencoders (VAEs) \citep{kingma2013auto} to Generative Adversarial Networks (GANs) \citep{goodfellow2014generative}, and more recently to diffusion models \citep{ho2020denoising}.
These generative models progressively improve the realism of synthetic images, effectively blurring the boundary between real and artificial content.
This progress creates new challenges for detection, as even humans struggle to distinguish AI-generated images from authentic ones, achieving only around 62\% accuracy \citep{roca2025good, lu2023seeing}.
Such limitations motivate the development of automated methods capable of reliably identifying synthetic media.

Early detection approaches focused on identifying artifacts inherent to generative models.
These artifacts include inconsistencies in pixel patterns, unnatural textures, irregular noise distributions, and subtle distortions in geometry or lighting \citep{frank2020leveraging, wang2020cnn}.
As generative models have become more sophisticated, deep learning classifiers have been increasingly applied to detect AI-generated images \citep{nataraj2019detecting, bird2024cifake}.
Hybrid forensic systems, combining deep learning with traditional forensic techniques, have further improved detection effectiveness \citep{yan2025sanitycheckaigeneratedimage, radford2021learning} 

Several novel methods emerge to address specific challenges in detection.
DIRE utilizes reconstruction errors derived from diffusion model inversion to detect AI-generated images \citep{wang2023dire}.
Similarly, SSP shows that even a single, carefully selected image patch can suffice for accurate detection, highlighting the presence of localized artifacts \citep{chen2024single}, and GANs have been found to generally have identifying artifacts/fingerprints \citep{yu2019iccv, durall2020cvpr, corvi2023}.
Recent research also explores leveraging multimodal large language models, which provide visually grounded explanations for detection decisions and enhance interpretability \citep{ojha2023towards}.
These advances collectively illustrate the rapid evolution of detection strategies, reflecting the ongoing arms race between generative models and discriminators. 

Benchmarking plays a crucial role in evaluating detector performance.
Large-scale datasets such as GenImage \citep{zhu2023genimagemillionscalebenchmarkdetecting} and Chameleon \citep{yan2025sanitycheckaigeneratedimage} provide diverse evaluation scenarios across a wide range of generative models, including Stable Diffusion \citep{rombach2022high}, Midjourney \citep{midjourney}, and BigGAN \citep{brock2018large}.
These benchmarks assess not only detection accuracy but also robustness under real-world conditions, such as low-resolution images, compression artifacts, and blurring.
Analyses from GenImage indicate that higher-resolution images reveal finer, more detectable artifacts, improving detection performance, whereas low-resolution or compressed images present greater challenges \citep{zhu2023genimagemillionscalebenchmarkdetecting}.

Some studies focus on leveraging semantic information, such as the number of fingers on a hand \citep{cheng2025co, zhang2023detecting}. 
At the same time, other work investigated how detectors designed and trained to detect GAN-generated images fail to generalize to diffusion-generated images and how detectors fail when images have been compressed \citep{corvi2023,gragnaniello2021}.
However, no work has investigated how dataset distribution and the resulting complexity influence detector performance.  

Kolmogorov complexity provides a framework to quantify intrinsic dataset complexity by measuring the length of the shortest program capable of reproducing it \citep{li2008introduction}.  
For image datasets, low Kolmogorov complexity corresponds to highly structured or repetitive content, which generators can learn easily, producing nearly indistinguishable synthetic images \citep{zenil2015two}.  

While previous studies have focused on developing AI-generated image detectors and evaluating them on large-scale benchmarks, none have explicitly analyzed the dynamics of the ongoing arms race between generators and discriminators over time.  
In this work, we are the first to systematically investigate how this long-term battle unfolds, identifying conditions under which detectors hold the greatest advantage.

\section{Methodology}
\label{sec:methodology}

We begin by formalizing the asymmetry of the detection problem.  
Let $p(x)$ denote the true data distribution and $q(x)$ the distribution induced by a generator.  
As long as $p \neq q$, there exists a discriminator $D$ with non-trivial accuracy in distinguishing real from synthetic samples.  
In the limiting case where $q(x) = p(x)$, however, the detection task becomes ill-posed: no discriminator can do better than always guessing the more probable class.   

\begin{proposition}
Let a dataset consist of real and generated samples with priors $\pi_r$ and $\pi_f = 1-\pi_r$.  
If the generator distribution $q(x)$ equals the data distribution $p(x)$ for all $x$, then every discriminator has accuracy equal to $\max\{\pi_r,\pi_f\}$.
\end{proposition}

\begin{proof}
Let $p(x)$ and $q(x)$ denote the densities of real and generated samples w.r.t.\ a common dominating measure.  
The mixture (marginal) density is
\begin{equation}
m(x) = \pi_r p(x) + \pi_f q(x).
\end{equation}
Conditioned on $X=x$, the posterior probabilities of the two classes are
\begin{equation}
\Pr(\text{real}\mid x) = \frac{\pi_r p(x)}{m(x)},\qquad
\Pr(\text{fake}\mid x) = \frac{\pi_f q(x)}{m(x)}.
\end{equation}
The Bayes-optimal classifier chooses the class with larger posterior probability. Hence, the pointwise probability of a correct decision given $x$ is
\begin{equation}
\max\left\{ \frac{\pi_r p(x)}{m(x)}, \frac{\pi_f q(x)}{m(x)} \right\}.
\end{equation}
The overall Bayes-optimal accuracy is
\begin{equation}
\mathrm{Acc}^\star = \int_{\mathcal X} \max\left\{\frac{\pi_r p(x)}{m(x)}, \frac{\pi_f q(x)}{m(x)}\right\} m(x)\, dx
= \int_{\mathcal X} \max\{\pi_r p(x), \pi_f q(x)\}\, dx.
\end{equation}
If $q(x)=p(x)$ everywhere, then $\max\{\pi_r p(x), \pi_f q(x)\} = \max\{\pi_r,\pi_f\} \, p(x)$, so
\begin{equation}
\mathrm{Acc}^\star = \max\{\pi_r,\pi_f\} \int_{\mathcal X} p(x)\, dx = \max\{\pi_r,\pi_f\}.
\end{equation}
Thus, when $q=p$, the Bayes-optimal accuracy equals the prior of the more probable class, and no discriminator can outperform this baseline.
\end{proof}
This establishes the theoretical limit of detection and motivates our investigation of the practical regimes where discriminators retain predictive power. 
In particular, we study how two dataset-dependent factors shape discriminator performance: (i) dataset complexity and (ii) input dimensionality.  
We use Kolmogorov complexity $K(\mathcal{D})$ as a conceptual measure of dataset complexity, and approximate it empirically using lossless compression.  
Dimensionality, in contrast, refers to the raw input dimension $d$ (e.g.\ the number of pixels per image).  
While larger $d$ expands the feature space in which distributions $p$ and $q$ can be separated, it also increases the sample complexity required for reliable discrimination.  
Our experiments thus explore the trade-off between these two factors across a wide range of datasets.   

\subsection{Diffusion Model and Discriminator Training}

For each dataset, a diffusion model is trained to generate synthetic samples.  
Synthetic subsets are denoted as $D_{\text{train}}^{\text{fake}}$, $D_{\text{val}}^{\text{fake}}$, and $D_{\text{test}}^{\text{fake}}$, with sizes matched to the corresponding real subsets.  
The discriminator is trained on $D_{\text{train}}^{\text{real}} \cup D_{\text{train}}^{\text{fake}}$ and validated on $D_{\text{val}}^{\text{real}} \cup D_{\text{val}}^{\text{fake}}$, while evaluation is performed on $D_{\text{test}}^{\text{real}} \cup D_{\text{test}}^{\text{fake}}$.  
The setup is shown in \Cref{fig: overview} and enables controlled comparisons of detection performance across datasets that differ in both complexity and resolution.  

We evaluate six discriminator configurations that differ in architecture and input representation.  
The \texttt{Base} discriminator is a compact convolutional neural network with approximately $40{,}000$ parameters.  
The \texttt{Big} discriminator is a deeper variant with about $520{,}000$ parameters, providing increased capacity while maintaining architectural similarity to the base model.  

Both the \texttt{Base} and \texttt{Big} discriminators are trained under two input modalities.  
In the \textit{Pixel} setting, and in the \textit{Fourier} setting. In the latter, a two-dimensional Fast Fourier Transform (FFT) is applied, and the log-magnitude spectrum of each channel is used as input.  
This results in four models: \texttt{Pixel-Base}, \texttt{Pixel-Big}, \texttt{Fourier-Base}, and \texttt{Fourier-Big}.  

In addition, we evaluate a pretrained discriminator based on ResNet-18 \citep{he2016deep} pretrained on ImageNet, comprising approximately $11$M parameters.  
Two training regimes are considered.  
In the \texttt{Linear-ResNet} setting, only the final classification layer is optimized while the ResNet backbone remains frozen.  
In the \texttt{Finetuned-Resnet} setting, the entire network is updated end-to-end.  

\subsection{Approximation of Dataset Complexity}

Kolmogorov complexity is not computable due to the undecidability of the halting problem\citep{chaitin1995program}, but compression-based methods provide a tractable and meaningful approximation \citep{li2008introduction}.  
Modern lossless compressors exploit redundancies and regularities in the data, yielding an effective upper bound on true Kolmogorov complexity \citep{grunwald2004shannon}.  
Datasets that compress strongly exhibit high internal structure, whereas datasets that compress poorly contain greater variability.  

We quantify the \textit{Complexity} of a dataset $D$ using its \textit{Compression Ratio}:
\begin{equation}
    C(D) = \frac{S_{\text{comp}}(D)}{S_{\text{orig}}(D)},
\end{equation}
where $S_{\text{orig}}(D)$ is the size of the dataset in bytes (raw NumPy representation) and $S_{\text{comp}}(D)$ is the size after compression.  
All images are concatenated into a single PNG file prior to compression to maximize exploitation of spatial redundancies \citep{w3cpngencoders}.  
This compressed size then serves as a practical proxy for Kolmogorov complexity  \citep{grunwald2004shannon}.  

\paragraph{Choice of compressor.}  
While several compression algorithms could be employed, we adopt PNG \citep{w3cpngencoders} as our primary measure due to its widespread use and high optimization.  
For robustness, we also computed dataset complexity using \texttt{zip}, \texttt{bzip2}, \texttt{zstd}, and NumPy's \texttt{npz} (Table~\ref{tab:dataset_complexity_merged}).  
The relative ranking of datasets remained highly consistent across methods (Spearman $\rho \approx 0.80$–$0.95$), confirming that PNG compression serves as a reliable proxy for dataset complexity.

\section{Experiments}

We evaluate the detectability of AI-generated images along two primary axes: dataset complexity and image resolution.  
For the complexity experiments, all datasets are zero-padded or resized to a common resolution of $32 \times 32$ pixels.  
For the resolution experiments, we vary the input resolution using the OrganAMNIST dataset \citep{medmnistv2,medmnistv1} as a case study.  
To capture a broad spectrum of dataset complexity, we consider 19 datasets drawn from diverse distributions.  
The complete list of datasets is provided in Table~\ref{tab:dataset_size}.  
 
To ensure comparability across datasets, a consistent preprocessing and compression pipeline is applied to all datasets used for complexity evaluation.  
Diffusion models are trained for each dataset using a standardized configuration appropriate for $32 \times 32$ resolution.  
Apart from controlled experimental variations such as image size or dataset, architecture depth and channel width are kept constant across training runs.  
We employ a conditional U-Net backbone trained with the standard DDPM (Denoising Diffusion Probabilistic Model) formulation \citep{ho2020denoising}.  
The U-Net consists of 3 levels of encoder–decoder with symmetric skip connections and self-attention blocks at multiple resolutions to capture both local and global dependencies. 
The base channel width is 128 and doubles after each level.

Each dataset is split into training ($D_{\text{train}}^{\text{real}}$), validation ($D_{\text{val}}^{\text{real}}$), and test ($D_{\text{test}}^{\text{real}}$) subsets.  
When predefined validation and test splits with reasonable sizes are available, they are preserved.  
Otherwise, all available images are pooled and randomly partitioned.  
Validation and test sets each contain either $10{,}000$ images or one eighth of the total dataset size, whichever is smaller, with the remainder used for training.  
For datasets with limited sample counts, standard data augmentations such as horizontal and vertical flips are applied to increase effective sample size.

For each dataset, a diffusion model is trained to produce synthetic samples, resulting in $D_{\text{train}}^{\text{fake}}$, $D_{\text{val}}^{\text{fake}}$, and $D_{\text{test}}^{\text{fake}}$ splits that mirror the sizes of their real counterparts.  
Discriminators are trained using $D_{\text{train}}^{\text{real}} \cup D_{\text{train}}^{\text{fake}}$, validated on $D_{\text{val}}^{\text{real}} \cup D_{\text{val}}^\text{{fake}}$, and evaluated on $D_{\text{test}}^{\text{real}} \cup D_{\text{test}}^{\text{fake}}$.  
This setup enables a controlled comparison of detection performance across datasets of varying complexity and dimensionality.  

Each of the six discriminator variants (see \Cref{sec:methodology}) is trained and evaluated independently.  
For each variant, experiments are repeated five times with different random seeds, and the mean detection performance is reported.

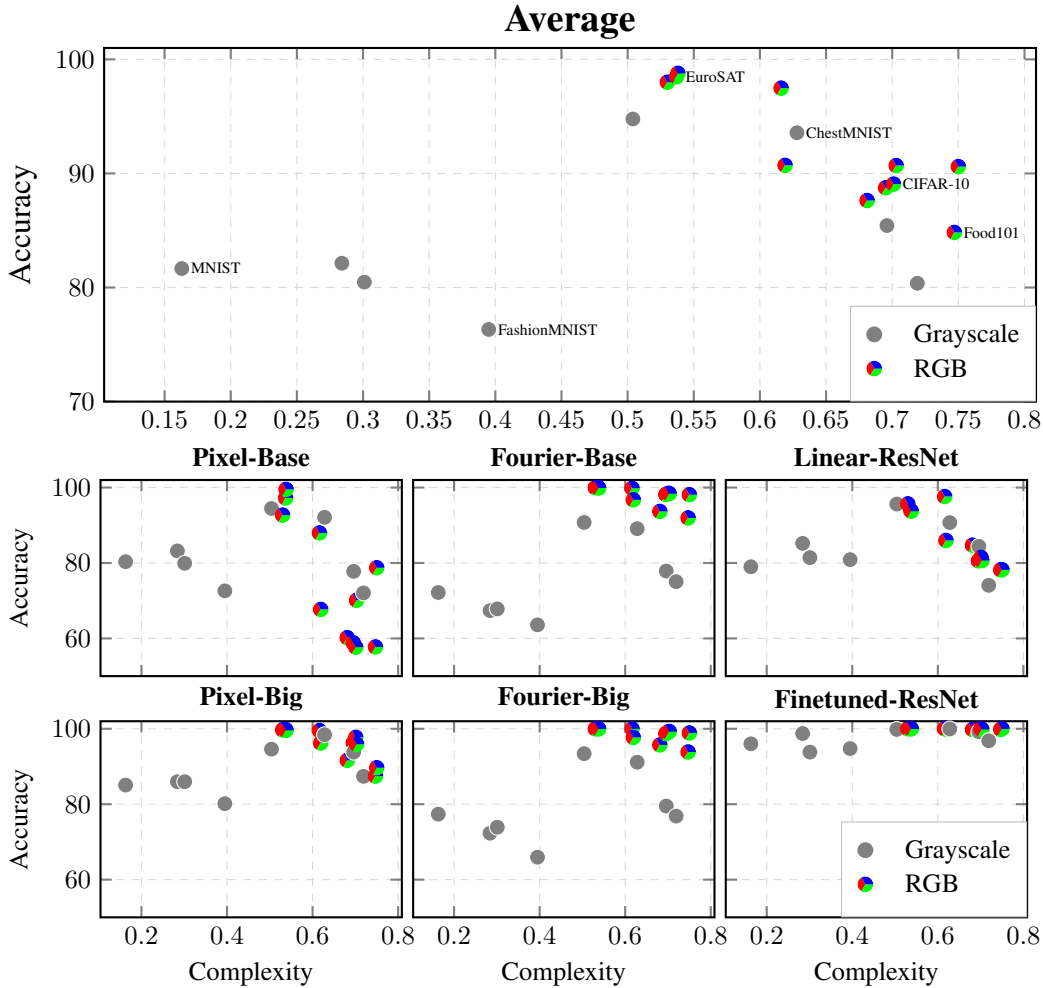
\begin{figure}[t]
    \centering
    \begin{tikzpicture}
    \pgfdeclareplotmark{RGB dot}{%
        \pgfpathcircle{\pgfpointorigin}{3pt}%
        \pgfusepath{clip}%
        \pgfsetfillcolor{red}%
        \pgfpathmoveto{\pgfpointorigin}%
        \pgfpathlineto{\pgfpointpolar{120}{3pt}}%
        \pgfpatharc{120}{240}{3pt}%
        \pgfpathclose%
        \pgfusepath{fill}%
        \pgfsetfillcolor{green}%
        \pgfpathmoveto{\pgfpointorigin}%
        \pgfpathlineto{\pgfpointpolar{240}{3pt}}%
        \pgfpatharc{240}{360}{3pt}%
        \pgfpathclose%
        \pgfusepath{fill}%
        \pgfsetfillcolor{blue}%
        \pgfpathmoveto{\pgfpointorigin}%
        \pgfpathlineto{\pgfpointpolar{0}{3pt}}%
        \pgfpatharc{0}{120}{3pt}%
        \pgfpathclose%
        \pgfusepath{fill}%
    }
    
    \begin{axis}[
        width=\linewidth,
        height=0.45\linewidth,
        ymin=70,
        ymax=101,
        grid=both,
        grid style={dashed, gray!30},
        scatter/classes={
            1={mark=*,mark size=3pt,fill=gray,draw=white}, 
            3={mark=RGB dot,mark size=3pt}      
        },
        axis line style={thick},
        tick style={thick},
        y dir=normal,
        title={}, 
        ylabel={Accuracy}, 
        ylabel style={anchor=east},
        yticklabel pos=left,
        axis y line*=left,
        axis lines=box,
        xlabel style={font=\large},
        ylabel style={font=\large},
    ]
    
        \addplot+[
            scatter, only marks, scatter src=explicit symbolic
        ] table [x=complexity, y=Average, meta=channels, col sep=comma]
          {figures/plots/accuracies/accuracy_summary.csv};
          
        \node[anchor=west,font=\tiny] at (axis cs:0.747,84.8 ) {Food101};
        \node[anchor=west,font=\tiny] at (axis cs:0.701,89.1 ) {CIFAR-10};
        \node[anchor=west,font=\tiny] at (axis cs:0.163,81.7 ) {MNIST};
        \node[anchor=west,font=\tiny] at (axis cs:0.395,76.3 ) {FashionMNIST};
        \node[anchor=west,font=\tiny] at (axis cs:0.628,93.6 ) {ChestMNIST};
        \node[anchor=west,font=\tiny] at (axis cs:0.537,98.5 ) {EuroSAT};

    \end{axis}

    \node[
        draw=gray!50,
        fill=white,
        inner sep=6pt,
        anchor=south east,
        at={(current axis.south east)}
    ] {
        \begin{tabular}{@{}cl@{}}
            \tikz{\fill[gray] (0,0) circle (3pt);} & Grayscale \\[2pt]
            \tikz{
                \pgfpathcircle{\pgfpointorigin}{3pt}%
                \pgfusepath{clip}%
                \pgfsetfillcolor{red}%
                \pgfpathmoveto{\pgfpointorigin}%
                \pgfpathlineto{\pgfpointpolar{120}{3pt}}%
                \pgfpatharc{120}{240}{3pt}%
                \pgfpathclose%
                \pgfusepath{fill}%
                \pgfsetfillcolor{green}%
                \pgfpathmoveto{\pgfpointorigin}%
                \pgfpathlineto{\pgfpointpolar{240}{3pt}}%
                \pgfpatharc{240}{360}{3pt}%
                \pgfpathclose%
                \pgfusepath{fill}%
                \pgfsetfillcolor{blue}%
                \pgfpathmoveto{\pgfpointorigin}%
                \pgfpathlineto{\pgfpointpolar{0}{3pt}}%
                \pgfpatharc{0}{120}{3pt}%
                \pgfpathclose%
                \pgfusepath{fill}%
            } & RGB \\
        \end{tabular}
    };
    \node[anchor=south, font=\Large\bfseries] at (current axis.north) {\textbf{Average}};
    \end{tikzpicture}
    \begin{tikzpicture}
    \pgfdeclareplotmark{RGB dot}{%
        \pgfpathcircle{\pgfpointorigin}{3pt}%
        \pgfusepath{clip}%
        \pgfsetfillcolor{red}%
        \pgfpathmoveto{\pgfpointorigin}%
        \pgfpathlineto{\pgfpointpolar{120}{3pt}}%
        \pgfpatharc{120}{240}{3pt}%
        \pgfpathclose%
        \pgfusepath{fill}%
        \pgfsetfillcolor{green}%
        \pgfpathmoveto{\pgfpointorigin}%
        \pgfpathlineto{\pgfpointpolar{240}{3pt}}%
        \pgfpatharc{240}{360}{3pt}%
        \pgfpathclose%
        \pgfusepath{fill}%
        \pgfsetfillcolor{blue}%
        \pgfpathmoveto{\pgfpointorigin}%
        \pgfpathlineto{\pgfpointpolar{0}{3pt}}%
        \pgfpatharc{0}{120}{3pt}%
        \pgfpathclose%
        \pgfusepath{fill}%
    }
    \pgfplotsset{every axis/.append style={
        width=0.4\linewidth,
        height=0.3\linewidth,
        ymin=50,
        ymax=102,
        grid=both,
        grid style={dashed, gray!30},
        scatter/classes={
            1={mark=*,mark size=3pt,fill=gray,draw=white}, 
            3={mark=RGB dot,mark size=3pt}      
        },
        axis line style={thick},
        tick style={thick},
        y dir=normal
    }}
    \begin{groupplot}[
        group style={
            group size=3 by 2,
            horizontal sep=0.15cm,
            vertical sep=0.6cm
        }
    ]
    \nextgroupplot[
        title={}, ylabel={Accuracy}, ylabel style={anchor=east},
        xlabel={},
        yticklabel pos=left,
        axis y line*=left,
        axis lines=box,
        xticklabels={}
    ]
        \addplot+[
            scatter, only marks, scatter src=explicit symbolic
        ] table [x=complexity, y=avg_accuracy, meta=channels, col sep=comma]
          {figures/plots/accuracies/Base_accuracy_stats.csv};
    
    \nextgroupplot[
        title={}, ylabel={},
        xlabel={},
        axis lines=box,
        yticklabels={},
        ymajorgrids=true,
        xticklabels={}
    ]
        \addplot+[
            scatter, only marks, scatter src=explicit symbolic
        ] table [x=complexity, y=avg_accuracy, meta=channels, col sep=comma]
          {figures/plots/accuracies/BaseFourier_accuracy_stats.csv};
    
    \nextgroupplot[
        title={}, ylabel={},
        xlabel={},
        axis lines=box,
        yticklabels={},
        ymajorgrids=true,
        xticklabels={}
    ]
        \addplot+[
            scatter, only marks, scatter src=explicit symbolic
        ] table [x=complexity, y=avg_accuracy, meta=channels, col sep=comma]
          {figures/plots/accuracies/ResNet_accuracy_stats.csv};
    
    \nextgroupplot[
        title={}, ylabel={Accuracy}, ylabel style={anchor=east},
        xlabel={Complexity},
        yticklabel pos=left,
        axis y line*=left,
        axis lines=box
    ]
        \addplot+[
            scatter, only marks, scatter src=explicit symbolic
        ] table [x=complexity, y=avg_accuracy, meta=channels, col sep=comma]
          {figures/plots/accuracies/Big_accuracy_stats.csv};
    
    \nextgroupplot[
        title={}, ylabel={},
        xlabel={Complexity},
        axis lines=box,
        yticklabels={},
        ymajorgrids=true
    ]
        \addplot+[
            scatter, only marks, scatter src=explicit symbolic
        ] table [x=complexity, y=avg_accuracy, meta=channels, col sep=comma]
          {figures/plots/accuracies/BigFourier_accuracy_stats.csv};
    
    \nextgroupplot[
        title={}, ylabel={},
        xlabel={Complexity},
        axis lines=box,
        yticklabels={},
        ymajorgrids=true
    ]
        \addplot+[
            scatter, only marks, scatter src=explicit symbolic
        ] table [x=complexity, y=avg_accuracy, meta=channels, col sep=comma]
          {figures/plots/accuracies/ResNetFine_accuracy_stats.csv};
    
    \end{groupplot}
    
    \node[
        draw=gray!50,
        fill=white,
        inner sep=6pt,
        anchor=south east,
        at={(group c3r2.south east)}
    ] {
        \begin{tabular}{@{}cl@{}}
            \tikz{\fill[gray] (0,0) circle (3pt);} & Grayscale \\[2pt]
            \tikz{
                \pgfpathcircle{\pgfpointorigin}{3pt}%
                \pgfusepath{clip}%
                \pgfsetfillcolor{red}%
                \pgfpathmoveto{\pgfpointorigin}%
                \pgfpathlineto{\pgfpointpolar{120}{3pt}}%
                \pgfpatharc{120}{240}{3pt}%
                \pgfpathclose%
                \pgfusepath{fill}%
                \pgfsetfillcolor{green}%
                \pgfpathmoveto{\pgfpointorigin}%
                \pgfpathlineto{\pgfpointpolar{240}{3pt}}%
                \pgfpatharc{240}{360}{3pt}%
                \pgfpathclose%
                \pgfusepath{fill}%
                \pgfsetfillcolor{blue}%
                \pgfpathmoveto{\pgfpointorigin}%
                \pgfpathlineto{\pgfpointpolar{0}{3pt}}%
                \pgfpatharc{0}{120}{3pt}%
                \pgfpathclose%
                \pgfusepath{fill}%
            } & RGB \\
        \end{tabular}
    };
    
    \node at ($(group c1r1.north) + (0,0.3cm)$) {\textbf{Pixel-Base}};
    \node at ($(group c2r1.north) + (0,0.3cm)$) {\textbf{Fourier-Base}};
    \node at ($(group c3r1.north) + (0,0.3cm)$) {\textbf{Linear-ResNet}};
    \node at ($(group c1r2.north) + (0,0.3cm)$) {\textbf{Pixel-Big}};
    \node at ($(group c2r2.north) + (0,0.3cm)$) {\textbf{Fourier-Big}};
    \node at ($(group c3r2.north) + (0,0.3cm)$) {\textbf{Finetuned-ResNet}};
    
    \end{tikzpicture}
    \caption{Discriminator accuracy across dataset complexity. \textbf{Top:} Overview across all datasets and models, gray points indicate grayscale images and tri-color points indicate RGB images. Medium-complexity datasets are easiest to detect, while simple datasets are nearly perfectly reproduced, and complex datasets mask generator errors. \textbf{Bottom:} Model-specific performance breakdown. Overall, increasing model capacity improves performance, particularly on high-complexity datasets. Fourier preprocessing boosts detection for RGB datasets, while fine-tuning ResNets achieves near-perfect accuracy across most datasets. Diminishing returns are observed when combining large models with Fourier preprocessing, and low-resolution grayscale datasets benefit less from these enhancements.}
    \label{fig:combined_accuracy}
\end{figure}

\section{Results}
\subsection{Results Across Dataset Complexity}

\Cref{fig:combined_accuracy} (top) summarizes the relationship between dataset complexity and discriminator accuracy.
Detailed results for each dataset are reported in \Cref{tab:accuracy_comparison}.
The figure shows average performance across all discriminator architectures.

At low complexity, the generator can capture the distribution almost perfectly, making very few mistakes, and the discriminator has a hard task.
At high complexity, the data distribution is wide, and the discriminator cannot reliably distinguish mistakes from genuine variability.
At medium complexity, the discriminator excels: the generator struggles to learn the distribution, producing systematic errors, while the dataset’s diversity is limited enough that these errors are clearly detectable.

\subsection{Discriminator Capacity}

Figure~\ref{fig:combined_accuracy} (bottom) summarizes discriminator performance across datasets of varying complexity.
Per-dataset results are given in \Cref{tab:accuracy_comparison}.
All models struggle on low-complexity datasets, where generators closely match the real distribution.
Accuracy improves on intermediate-complexity datasets, where imperfections are more visible.
For high-complexity datasets, smaller models decline, while larger ones retain accuracy by capturing subtler inconsistencies.

Increasing model capacity consistently improves performance.  
Transitioning from \texttt{Base} to \texttt{Big} CNNs boosts accuracy, particularly in high-complexity regimes.  
Fourier preprocessing stabilizes training and enhances detection for RGB datasets, as artifacts are spread in the frequency domain.  
However, combining Fourier transforms with larger CNNs yields only marginal additional improvement, suggesting diminishing returns when both capacity and preprocessing are maximized.  
For low-resolution grayscale datasets, Fourier preprocessing offers minimal benefit.  

\texttt{ResNet} discriminators follow similar trends.
Pretrained \texttt{ResNets} with linear projection perform well on intermediate datasets but struggle on the simplest and most complex ones.
Fine-tuning boosts accuracy across most datasets, with more errors on low-complexity and fewer on complex datasets, showing generators still make detectable mistakes.
These results show that model capacity is key for complex datasets.
As complexity rises, discriminators need more expressive architectures, but practical constraints in large-scale applications emphasize the need for efficient, high-capacity models.
In practical applications, such as content moderation at scale, computational constraints limit the extent to which discriminators can be enlarged, highlighting the need for efficient yet high-capacity models.

\subsection{Results Across Dataset Resolution}\label{sec: results across resolution}

\begin{figure}[t]
    \centering
    \begin{tikzpicture}
    \begin{axis}[
        width=0.75\textwidth,
        height=5cm,
        xlabel={Image Resolution},
        ylabel={Accuracy (\%)},
        ymin=65,
        ymax=101,
        xmin=28,
        xmax=140,
        xtick={32,64,128},
        xticklabels={32×32,64×64,128×128},
        xmode=log,
        log basis x={2},
        grid=major,
        grid style={dashed, gray!20},
        legend pos=south east,
        legend style={
            at={(1.02,0.5)},    
            anchor=west,        
            font=\small,
            fill=white,
            fill opacity=0.9,
            draw=gray!50,
            rounded corners=2pt,
            inner sep=4pt,
            legend columns=1    
        },
        axis line style={thick},
        tick style={thick},
        xlabel style={font=\large},
        ylabel style={font=\large},
    ]
    
    \addplot[
        color=green!80!black,
        mark=triangle*,
        mark size=3pt,
        line width=1pt,
        dashed,
        mark options={fill=green!80!black, solid}
    ] coordinates {
        (32,70.08)
        (64,69.53)
        (128,91.78)
    };
    \addlegendentry{Pixel-Base}
    
    \addplot[
        color=green!80!black,
        mark=triangle*,
        mark size=3pt,
        line width=1pt,
        mark options={fill=green!80!black}
    ] coordinates {
        (32,86.23)
        (64,89.08)
        (128,99.41)
    };
    \addlegendentry{Pixel-Big}
    
    \addplot[
        color=red!80!black,
        mark=diamond*,
        mark size=3pt,
        line width=1pt,
        dashed,
        mark options={fill=red!80!black, solid}
    ] coordinates {
        (32,74.68)
        (64,82.4)
        (128,97.1)
    };
    \addlegendentry{Fourier-Base}
    
    \addplot[
        color=red!80!black,
        mark=diamond*,
        mark size=3pt,
        line width=1pt,
        mark options={fill=red!80!black}
    ] coordinates {
        (32,76.54)
        (64,86.28)
        (128,98.27)
    };
    \addlegendentry{Fourier-Big}
    
    \addplot[
        color=blue!80!black,
        mark=square*,
        mark size=3pt,
        line width=1pt,
        dashed,
        mark options={fill=blue!80!black, solid}
    ] coordinates {
        (32,74.05)
        (64,80.55)
        (128,93.2)
    };
    \addlegendentry{Linear-ResNet}
    
    \addplot[
        color=blue!80!black,
        mark=square*,
        mark size=3pt,
        line width=1pt,
        mark options={fill=blue!80!black,}
    ] coordinates {
        (32,97.53)
        (64,97.15)
        (128,99.7)
    };
    \addlegendentry{Finetuned-ResNet}
    
    \node[
        pin={[pin edge={thick,red!70!black},pin distance=1.5cm]45:{Best: 99.7\%}},
        circle,
        fill=red!70!black,
        inner sep=1pt
    ] at (axis cs:128,99.7) {};
    
    \end{axis}
    \end{tikzpicture}
    \caption{Classification accuracy comparison across different model architectures on the OrganAMNIST dataset. Results show performance scaling with image resolution from 32×32 to 128×128 pixels. Fine-tuned ResNet consistently achieves the highest accuracy across all resolutions.}
    \label{fig:organamnist_accuracy}
\end{figure}
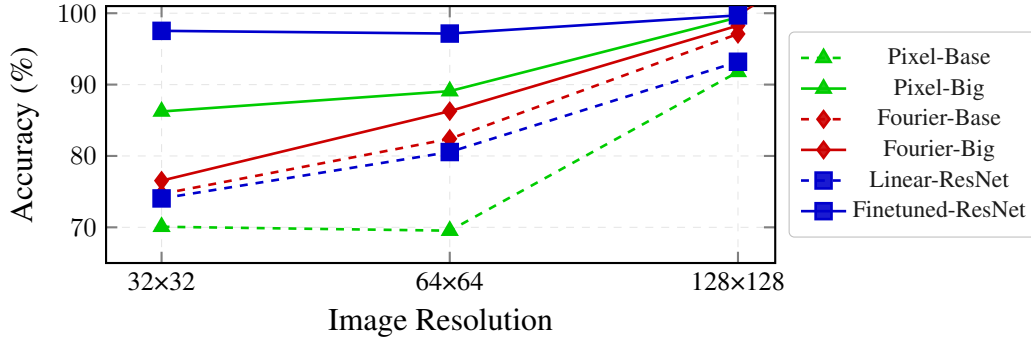


For the multi-resolution experiment, we used the OrganAMNIST dataset and evaluated diffusion-generated images at three resolutions: $32\times 32$, $64\times 64$ and $128\times 128$ pixels.  
At $32\times 32$, discriminator accuracy is comparatively low across most architectures, suggesting that low-resolution generations obscure many artifacts and are therefore harder to classify as fake.  
The fine-tuned ResNet, however, already achieves relatively strong performance at this resolution and continues to improve as resolution increases.  
At $64\times 64$, performance improves slightly across models, indicating that increased resolution exposes additional cues for detection.  
At $128\times 128$, all models achieve very high accuracy, showing that higher resolutions amplify detectable differences from real data, likely due to high-frequency artifacts introduced by the generator.  
These results are consistent with previous studies showing that higher resolution makes it easier to detect synthetic images \citep{zhu2023genimagemillionscalebenchmarkdetecting}.

This highlights a key dynamic in the generator-discriminator arms race: the generator's struggle to maintain high fidelity at larger scales, evidenced by the rising FID score (\Cref{fig:diffusion_images}), provides a clearer signal for detectors. The high-frequency artifacts that degrade the generator's performance appear to be the very same cues that enable discriminators to achieve near-perfect accuracy at high resolution.

\newcommand{\ImagePair}[3]{%
  \begin{tabular}{@{}c@{}} 
    {\scriptsize \textbf{#3}} \\[-0.1em]
    \includegraphics[width=\PairWidth]{#1} \\[-0.3em]
    \includegraphics[width=\PairWidth]{#2}
  \end{tabular}%
}

\newlength{\PairWidth}
\setlength{\PairWidth}{0.12\textwidth} 

\begin{figure}[t]
    \begin{center}
    \setlength{\tabcolsep}{2pt} 
    \renewcommand{\arraystretch}{1.0}
    \begin{tabular}{*{7}{c}} 
      \ImagePair{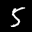}{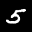}{MNIST} &
      \ImagePair{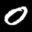}{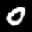}{EMNIST} &
      \ImagePair{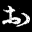}{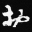}{KMNIST} &
      \ImagePair{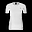}{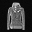}{FashionMNIST} &
      \ImagePair{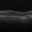}{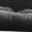}{OCTMNIST} &
      \ImagePair{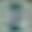}{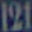}{SVHN} &
      \ImagePair{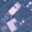}{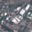}{EuroSAT} \\
      \ImagePair{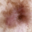}{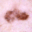}{DermaMNIST} &
      \ImagePair{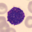}{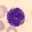}{BloodMNIST} &
      \ImagePair{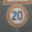}{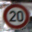}{GTSRB} &
      \ImagePair{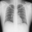}{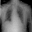}{ChestMNIST} &
      \ImagePair{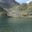}{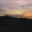}{Country211} &
      \ImagePair{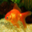}{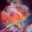}{CIFAR-100} &
      \ImagePair{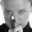}{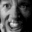}{FER-2013} \\
      \ImagePair{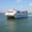}{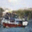}{CIFAR-10} &
      \ImagePair{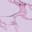}{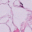}{PathMNIST} &
      \ImagePair{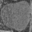}{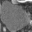}{OrganAMNIST} &
      \ImagePair{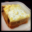}{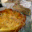}{Food101} &
      \ImagePair{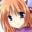}{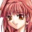}{Anime Face} &
        \multicolumn{2}{c}{%
        \begin{tabular}{c c}
            \toprule
            \textbf{Resolution} & \textbf{FID} $\downarrow$ \\
            \midrule
            $32 \times 32$ & 14.11 \\
            $64 \times 64$ & 24.08 \\
            $128 \times 128$ & 74.60 \\
            \bottomrule
        \end{tabular}
      } \\
    \end{tabular}
    
    \end{center}
    \caption{Real and generated images across datasets of varying complexity.  
    Each column shows a dataset, with real images on top and generated images on the bottom. 
    The bottom-right table reports FID scores across resolutions.  
    Generated images closely follow the real distributions, but FID increases with resolution, indicating that the diffusion model struggles to capture fine details.}

    \label{fig:diffusion_images}
\end{figure}

\subsection{Evaluation of the Diffusion Model}


Diffusion models are a leading approach for image generation due to their ability to capture complex data distributions and produce high-quality samples.
We evaluate the model qualitatively and quantitatively across datasets of varying complexity and resolution.

\textbf{Image quality across dataset complexity.}
\Cref{fig:diffusion_images} shows real images on top and generated images on the bottom for each dataset.
For simpler datasets such as FashionMNIST or KMNIST, generated samples are nearly indistinguishable from real images, reflecting low Real-AI FID values (see \Cref{tab:fid_results}).
In contrast, for more complex datasets like CIFAR-100, SVHN, or EuroSAT, subtle imperfections remain visible, corresponding to higher Real-AI FID scores.
These trends indicate that image quality decreases with increasing dataset complexity and help explain why discriminators perform better on more complex datasets, as subtle artifacts are easier to detect.

\textbf{Image quality across resolution.}
For OrganAMNIST, the model achieves strong fidelity at lower resolutions (see \Cref{fig:diffusion_images}, with performance gradually decreasing as resolution increases.
Overall, the diffusion model generates realistic and diverse samples for simpler datasets.
These results align with the discriminator performance shown in \Cref{sec: results across resolution} and \Cref{fig:organamnist_accuracy}; as resolution increases, discriminators become more effective at detecting fake samples.





\section{Limitations and Future Work}\label{sec: limitations}

Despite our study’s insights, several limitations suggest directions for future work.
The analysis is limited by the choice of generative model, including architecture and hyperparameters.
Diffusion model performance can vary with factors such as latent diffusion variants, number of diffusion steps, noise schedules, and other training settings, affecting image quality and artifact types.
Results may differ for other generative architectures or parameter configurations.
Extending the analysis to text-to-image models is also promising, as prompts introduce additional complexity that could influence discriminator performance.
Finally, complexity was measured at a fixed resolution of $32\times 32$ pixels, yet both dataset complexity and generative performance can scale with resolution, warranting further exploration.

Our study captures only a snapshot and does not consider the historical evolution of the generator-discriminator arms race. Studying past improvements could provide context for current detection challenges and reveal trends in model development. Also, human perception was not incorporated in the evaluation. Humans often serve as effective detectors of AI-generated content, so benchmarking against human performance could offer complementary insights.  

The measure of dataset complexity relies on standard PNG compression, which may not fully capture the intrinsic diversity of the data. 
For example, a dataset of random noise could appear highly complex under this metric, even though a learned compression model could efficiently encode it. Employing learned compression schemes tailored to each dataset could provide a more accurate assessment of structural complexity.  

Together, these limitations highlight both methodological constraints and opportunities for future research, including exploring higher-resolution images, alternative generative models, temporal dynamics, human-centered evaluations, and improved complexity metrics. Addressing these aspects would deepen our understanding of the conditions under which AI-generated content is most and least detectable.

\section{Conclusions}

Our study highlights the interplay between AI-generated image detectability, dataset complexity, data resolution, and discriminator capacity. Diffusion models are highly effective at learning simple datasets, producing images that closely match the real distribution. As dataset complexity increases, these models begin to make systematic errors, which discriminators can exploit to distinguish real from generated content. However, when datasets are extremely complex, even discriminators struggle to reliably detect fakes, as the diversity and variability in the data mask generator imperfections.  

Increasing data dimensionality, such as higher-resolution images, provides the discriminator with more features and subtle cues, improving detection accuracy. Larger discriminators further enhance performance, particularly in high-complexity regimes.
A key aspect we have only begun to explore is the synergistic effect of both complexity and resolution on detectability. As generators become more capable of producing high-resolution, complex images, the nature of the detectable artifacts may shift from global inconsistencies to subtle, high-frequency errors. This suggests that the "sweet spot" of intermediate complexity for detection may itself be resolution-dependent, a fascinating phenomenon that presents a rich and promising direction for further investigation to truly understand the boundaries of AI-generated content detection.

Looking forward, the rapid evolution of AI-generated content, including high-resolution images, text-to-image models, and multimodal media, presents both opportunities and challenges. Generators will continue to produce increasingly realistic content, while discriminators must adapt to maintain reliable detection. However, given the asymmetric nature of this arms race, it is likely that this battle will eventually be lost: as generative models approach perfect emulation of real data distributions, discriminators will be fundamentally limited in their ability to detect fakes. Understanding the limits of detection and the factors that influence it remains essential for building robust systems to safeguard digital media, mitigate misinformation, and preserve trust in online content. 
Our work provides a foundation for future research in this evolving landscape, guiding the development of both generative and discriminative AI in a responsible and informed manner.

\bibliographystyle{IEEEtran}
\bibliography{references}

\newpage

\section{Appendix}

\subsection{Extended Results}

\Cref{tab:accuracy_comparison} provides a comprehensive overview of classification accuracy across all datasets and discriminator architectures.  
The datasets span a wide range of intrinsic complexity, from simple digit datasets such as MNIST and KMNIST, to moderately complex datasets such as OCTMNIST and SVHN, to highly diverse datasets including CIFAR-10, PathMNIST, and Food101.  
The complexity column quantitatively reflects the structural richness and variability of each dataset, providing context for the observed performance trends.

\begin{table}[htbp]
\centering
\caption{Classification accuracy comparison across different models and datasets. 
The complexity column provides a quantitative measure of each dataset's structural richness and variability. 
Average accuracy across all models is also included for comparison. 
Fine-tuned ResNet consistently achieves the highest accuracy across datasets, demonstrating the importance of model capacity for handling complex and diverse image data.}
\label{tab:accuracy_comparison}
\resizebox{\textwidth}{!}{%
\begin{tabular}{lc|ccccccc}
\toprule
Dataset & Complexity & Average & Pixel-Base & Pixel-Big & Fourier-Base & Fourier-Big & Linear-ResNet & Finetuned-ResNet \\
\midrule
MNIST\citep{mnist} & 0.163 & \phantom{1}81.7 & \phantom{1}80.3 $\pm$ 2.7 & \phantom{1}\phantom{1}85.1 $\pm$ 4.9 & \phantom{1}72.2 $\pm$ 0.6 & \phantom{1}77.4 $\pm$ 0.6 & \phantom{1}79.0 $\pm$ 0.1 & \phantom{1}96.0 $\pm$ 3.9\\
EMNIST\citep{emnist} & 0.284 & \phantom{1}82.1 & \phantom{1}83.2 $\pm$ 1.0 & \phantom{1}86.0 $\pm$ 11.0 & \phantom{1}67.4 $\pm$ 0.8 & \phantom{1}72.3 $\pm$ 1.6 & \phantom{1}85.2 $\pm$ 0.1 & \phantom{1}98.7 $\pm$ 1.3\\
KMNIST\citep{torchvisiondatasets} & 0.301 & \phantom{1}80.5 & \phantom{1}79.9 $\pm$ 0.4 & \phantom{1}\phantom{1}86.0 $\pm$ 1.1& \phantom{1}67.8 $\pm$ 0.7 & \phantom{1}73.9 $\pm$ 1.4 & \phantom{1}81.4 $\pm$ 0.2 & \phantom{1}93.8 $\pm$ 3.8 \\
FashionMNIST\citep{fashionmnist} & 0.395 & \phantom{1}76.3 & \phantom{1}72.6 $\pm$ 1.2 & \phantom{1}\phantom{1}80.1 $\pm$ 3.1 & \phantom{1}63.6 $\pm$ 0.3 & \phantom{1}66.0 $\pm$ 0.5 & \phantom{1}80.9 $\pm$ 0.1 & \phantom{1}94.8 $\pm$ 4.4 \\
OCTMNIST\citep{medmnistv2,medmnistv1} & 0.504 & \phantom{1}94.8 & \phantom{1}94.5 $\pm$ 2.9 & \phantom{1}94.6 $\pm$ 10.4 & \phantom{1}90.8 $\pm$ 0.4 & \phantom{1}93.4 $\pm$ 0.4 & \phantom{1}95.6 $\pm$ 0.1 & \phantom{1}99.8 $\pm$ 0.0 \\
SVHN\citep{svhn} & 0.530 & \phantom{1}98.0 & \phantom{1}92.7 $\pm$ 7.1 & \phantom{1}\phantom{1}99.7 $\pm$ 0.1 & 100.0 $\pm$ 0.0 & 100.0 $\pm$ 0.0 & \phantom{1}95.7 $\pm$ 0.1 & 100.0 $\pm$ 0.0 \\
Eurosat-AUG\citep{eurosat} & 0.537 & \phantom{1}98.5 & \phantom{1}97.2 $\pm$ 4.5 & \phantom{1}\phantom{1}99.8 $\pm$ 0.2 & 100.0 $\pm$ 0.0 & 100.0 $\pm$ 0.0 & \phantom{1}93.9 $\pm$ 0.1 & 100.0 $\pm$ 0.0 \\
DermaMNIST-AUG\citep{medmnistv2,medmnistv1} & 0.538 & \phantom{1}98.8 & \phantom{1}99.5 $\pm$ 0.3 & \phantom{1}\phantom{1}99.6 $\pm$ 0.3 & \phantom{1}99.9 $\pm$ 0.0 & 100.0 $\pm$ 0.0 & \phantom{1}93.7 $\pm$ 0.2 & 100.0 $\pm$ 0.0 \\
BloodMNIST-AUG\citep{medmnistv2,medmnistv1} & 0.616 & \phantom{1}97.5 & \phantom{1}88.0 $\pm$ 5.7 & \phantom{1}\phantom{1}99.5 $\pm$ 0.7 & \phantom{1}99.8 $\pm$ 0.0 & 100.0 $\pm$ 0.0 & \phantom{1}97.6 $\pm$ 0.0 & 100.0 $\pm$ 0.0 \\
GTSRB\citep{stallkamp2011} & 0.619 & \phantom{1}90.7 & \phantom{1}67.7 $\pm$ 0.6 & \phantom{1}\phantom{1}96.2 $\pm$ 2.5 & \phantom{1}96.7 $\pm$ 0.2 & \phantom{1}97.7 $\pm$ 0.1 & \phantom{1}86.0 $\pm$ 0.2 & 100.0 $\pm$ 0.0 \\
ChestMNIST\citep{medmnistv2,medmnistv1} & 0.628 & \phantom{1}93.6 & \phantom{1}92.1 $\pm$ 5.6 & \phantom{1}\phantom{1}98.4 $\pm$ 0.6 & \phantom{1}89.1 $\pm$ 0.6 & \phantom{1}91.1 $\pm$ 0.3 & \phantom{1}90.7 $\pm$ 0.1 & \phantom{1}99.9 $\pm$ 0.0 \\
Country211\citep{torchvisiondatasets} & 0.681 & \phantom{1}87.6 & \phantom{1}60.2 $\pm$ 3.0 & \phantom{1}\phantom{1}91.6 $\pm$ 3.2 & \phantom{1}93.7 $\pm$ 0.3 & \phantom{1}95.7 $\pm$ 0.1 & \phantom{1}84.7 $\pm$ 0.2 & \phantom{1}99.8 $\pm$ 0.0 \\
CIFAR-100\citep{cifar10100} & 0.695 & \phantom{1}88.7 & \phantom{1}58.9 $\pm$ 1.1 & \phantom{1}\phantom{1}96.3 $\pm$ 1.3 & \phantom{1}98.2 $\pm$ 0.1 & \phantom{1}98.7 $\pm$ 0.0 & \phantom{1}80.6 $\pm$ 0.3 & \phantom{1}99.9 $\pm$ 0.0 \\
Fer-2013-AUG\citep{fer2013} & 0.696 & \phantom{1}85.4 & \phantom{1}77.8 $\pm$ 5.2 & \phantom{1}\phantom{1}93.8 $\pm$ 1.3 & \phantom{1}77.9 $\pm$ 0.2 & \phantom{1}79.5 $\pm$ 0.5 & \phantom{1}84.4 $\pm$ 0.3 & \phantom{1}99.2 $\pm$ 0.1 \\
CIFAR-10\citep{cifar10100} & 0.701 & \phantom{1}89.1 & \phantom{1}57.7 $\pm$ 1.5 & \phantom{1}\phantom{1}97.7 $\pm$ 0.6 & \phantom{1}98.5 $\pm$ 0.2 & \phantom{1}99.1 $\pm$ 0.1 & \phantom{1}81.5 $\pm$ 0.3 & \phantom{1}99.9 $\pm$ 0.0 \\
PathMNIST\citep{medmnistv2,medmnistv1} & 0.703 & \phantom{1}90.7 & \phantom{1}70.1 $\pm$ 0.6 & \phantom{1}\phantom{1}95.9 $\pm$ 2.3 & \phantom{1}98.4 $\pm$ 0.1 & \phantom{1}99.3 $\pm$ 0.1 & \phantom{1}80.6 $\pm$ 0.1 & \phantom{1}99.8 $\pm$ 0.1 \\
OrganAMNIST\citep{medmnistv2,medmnistv1} & 0.719 & \phantom{1}80.4 & \phantom{1}72.1 $\pm$ 1.5 & \phantom{1}\phantom{1}87.4 $\pm$ 4.2 & \phantom{1}75.0 $\pm$ 0.4 & \phantom{1}76.8 $\pm$ 0.3 & \phantom{1}74.1 $\pm$ 0.1 & \phantom{1}96.8 $\pm$ 1.1 \\
Food101\citep{food101} & 0.747 & \phantom{1}84.8 & \phantom{1}57.7 $\pm$ 3.0 & \phantom{1}\phantom{1}87.4 $\pm$ 4.6 & \phantom{1}92.0 $\pm$ 0.8 & \phantom{1}93.8 $\pm$ 0.4 & \phantom{1}78.2 $\pm$ 0.1 & \phantom{1}99.9 $\pm$ 0.0\\
Anime Face Dataset\citep{spencerchurchillbrianchao2019} & 0.750 & \phantom{1}90.6 & \phantom{1}78.7 $\pm$ 2.6 & \phantom{1}\phantom{1}89.7 $\pm$ 6.9 & \phantom{1}98.1 $\pm$ 0.2 & \phantom{1}98.9 $\pm$ 0.1 & \phantom{1}78.2 $\pm$ 0.1 & 100.0 $\pm$ 0.0 \\
\bottomrule
\end{tabular}
}
\end{table}

\subsection{Dataset Collection and Preparation}

The datasets (see Table \ref{tab:dataset_size}) used in this study are primarily obtained from \textit{Torchvision Datasets} \citep{torchvisiondatasets} and the \textit{MedMNIST+} collections \citep{medmnistv2, medmnistv1}.  
Table~\ref{tab:dataset_size} summarizes all datasets, including the number of training, validation, and test samples, as well as total sizes.  

To mitigate the limited number of samples in some datasets, data augmentation is applied.  
Augmented datasets are indicated by the suffix \textit{AUG}.  
EuroSAT and FER-2013 are augmented via horizontal flips.  
BloodMNIST and DermaMNIST are augmented using horizontal flips, vertical flips, and combined horizontal and vertical flips.  

To ensure consistent image dimensions, MNIST, EMNIST, KMNIST, and FashionMNIST ($28\times28$) are padded with black pixels to reach $32\times32$.  
All other datasets are either already at the target resolution or resized directly to $32\times32$.   

To study the effect of input dimensionality, we select OrganAMNIST because it is available at a high resolution of $128\times128$, which can be downscaled to $64\times64$ and $32\times32$ as needed.  

\begin{table}[ht]
\centering
\caption{Summary of datasets used in our study. Train, validation and test split are listed. Overview on the augmentation of smaller datasets is also provided.}
\begin{tabular}{lrrrr}
\toprule
\textbf{Dataset} & \textbf{Train} & \textbf{Val} & \textbf{Test} & \textbf{Total} \\
\midrule
MNIST\citep{mnist}  & 50,000 & 10,000 & 10,000 & 70,000 \\
EMNIST\citep{emnist}  & 102,800 & 10,000 & 10,000 & 122,800 \\
KMNIST\citep{torchvisiondatasets}  & 50,000 & 10,000 & 10,000 & 70,000 \\
FashionMNIST\citep{fashionmnist}  & 50,000 & 10,000 & 10,000 & 70,000 \\
OCTMNIST\citep{medmnistv2,medmnistv1}  & 89,309 & 10,000 & 10,000 & 109,309 \\
SVHN\citep{svhn}  & 79,289 & 10,000 & 10,000 & 99,289 \\
EuroSAT\citep{eurosat}  & 20,250 & 3,375 & 3,375 & 27,000 \\
EuroSAT-AUG  & 40,500 & 6,750 & 6,750 & 54,000 \\
DermaMNIST\citep{medmnistv2,medmnistv1}  & 7,513 & 1,251 & 1,251 & 10,015 \\
DermaMNIST-AUG  & 30,052 & 5,004 & 5,004 & 40,060 \\
BloodMNIST\citep{medmnistv2,medmnistv1}  & 12,820 & 2,136 & 2,136 & 17,092 \\
BloodMNIST-AUG  & 51,280 & 8,544 & 8,544 & 68,368 \\
GTSRB\citep{stallkamp2011}  & 38,881 & 6,479 & 6,479 & 51,839 \\
ChestMNIST\citep{medmnistv2,medmnistv1}  & 92,120 & 10,000 & 10,000 & 112,120 \\
Country211\citep{torchvisiondatasets}  & 47,476 & 7,912 & 7,912 & 63,300 \\
CIFAR-100\citep{cifar10100}  & 40,000 & 10,000 & 10,000 & 60,000 \\
FER-2013\citep{fer2013}  & 25,887 & 5,000 & 5,000 & 35,887 \\
FER-2013-AUG  & 51,774 & 10,000 & 10,000 & 71,774 \\
CIFAR-10\citep{cifar10100}  & 40,000 & 10,000 & 10,000 & 60,000 \\
PathMNIST\citep{medmnistv2,medmnistv1}  & 87,180 & 10,000 & 10,000 & 107,180 \\
OrganAMNIST\citep{medmnistv2,medmnistv1}  & 44,124 & 7,353 & 7,353 & 58,830 \\
Food-101\citep{food101}  & 81,000 & 10,000 & 10,000 & 101,000 \\
Anime Face Dataset\citep{spencerchurchillbrianchao2019} & 47,675 & 7,945 & 7,945 & 63,565 \\
\bottomrule
\end{tabular} \label{tab:dataset_size}
\end{table}

\subsubsection{Dataset Complexity}

To quantify the intrinsic complexity of the datasets, we approximate Kolmogorov complexity using compression-based measures.  
Each dataset is concatenated into a single PNG file, and the resulting compression ratio serves as a practical proxy for complexity.  
Datasets that compress efficiently exhibit more regularity and lower complexity, whereas datasets that compress poorly contain higher variability and are considered more complex.  

Table~\ref{tab:dataset_complexity_merged} reports the complexity values for all datasets, obtained both from PNG concatenation and alternative compression methods (\texttt{Zip}, \texttt{bzip2}, \texttt{Zstd}, \texttt{NumPy NPZ}).  
Lower values correspond to simpler datasets such as MNIST, while higher values correspond to more complex datasets such as Food-101 and Anime.  

\begin{table}[ht]
\centering
\caption{Complexity values of datasets obtained using PNG concatenation as well as alternative compression methods (PNG folder compression, Zip, bzip2, Zstd, and NumPy NPZ).  
The values serve as proxies for the intrinsic complexity of each dataset, with lower values indicating simpler, more regular datasets and higher values indicating more complex, diverse datasets.}
\begin{tabular}{l *{6}{c}}
\toprule
\textbf{Dataset} & \textbf{PNG concat.} & \textbf{PNG Folder} & \textbf{Zip of PNG} & \textbf{bzip2} & \textbf{Zstd} & \textbf{ NPZ} \\
\midrule
MNIST\citep{mnist} & 0.16 & 0.27 & 0.37 & 0.13 & 0.16 & 0.16 \\
EMNIST\citep{emnist} & 0.28 & 0.41 & 0.50 & 0.18 & 0.24 & 0.25 \\
KMNIST\citep{torchvisiondatasets} & 0.30 & 0.43 & 0.53 & 0.27 & 0.31 & 0.31 \\
FashionMNIST\citep{fashionmnist}  & 0.40 & 0.51 & 0.61 & 0.39 & 0.43 & 0.44 \\
OCTMNIST\citep{medmnistv2,medmnistv1}  & 0.50 & 0.62 & 0.72 & 0.50 & 0.68 & 0.66 \\
SVHN\citep{svhn}  & 0.53 & 0.56 & 0.59 & 0.79 & 0.95 & 0.88 \\
EuroSAT\citep{eurosat}  & 0.54 & 0.56 & 0.60 & 0.59 & 0.76 & 0.73 \\
DermaMNIST\citep{medmnistv2,medmnistv1}  & 0.54 & 0.57 & 0.60 & 0.77 & 0.94 & 0.87 \\
BloodMNIST\citep{medmnistv2,medmnistv1}  & 0.62 & 0.66 & 0.70 & 0.58 & 0.80 & 0.79 \\
GTSRB\citep{stallkamp2011} & 0.62 & 0.65 & 0.68 & 0.71 & 0.86 & 0.83 \\
ChestMNIST\citep{medmnistv2,medmnistv1}  & 0.63 & 0.74 & 0.84 & 0.75 & 0.98 & 0.94 \\
Country211\citep{torchvisiondatasets} & 0.68 & 0.71 & 0.74 & 0.85 & 0.97 & 0.90 \\
CIFAR-100\citep{cifar10100}  & 0.70 & 0.73 & 0.76 & 0.86 & 0.96 & 0.91 \\
FER-2013\citep{fer2013}  & 0.70 & 0.82 & 0.92 & 0.81 & 0.97 & 0.97 \\
CIFAR-10\citep{cifar10100}  & 0.70 & 0.74 & 0.77 & 0.86 & 0.97 & 0.92 \\
PathMNIST\citep{medmnistv2,medmnistv1}  & 0.70 & 0.73 & 0.76 & 0.63 & 0.84 & 0.81 \\
OrganAMNIST\citep{medmnistv2,medmnistv1}  & 0.72 & 0.84 & 0.94 & 0.76 & 0.88 & 0.89 \\
Food-101\citep{food101}  & 0.75 & 0.78 & 0.82 & 0.91 & 1.00 & 0.97 \\
Anime Face Dataset\citep{spencerchurchillbrianchao2019}  & 0.75 & 0.82 & 0.85 & 0.89 & 0.97 & 0.95 \\
\bottomrule
\end{tabular}
\label{tab:dataset_complexity_merged}
\end{table}

Random removal of samples from a dataset changes its size and compressed size but does not affect the underlying data distribution and maintains the same Compression Ratio.  
Formally, if $D = \{x_1, \dots, x_N\}$ and $D' \subset D$ is obtained by removing $k$ samples, the expected empirical distribution of $D'$ satisfies
\[
\mathbb{E}[\hat{p}_{D'}(x)] = \hat{p}_D(x),
\]
demonstrating that the underlying distribution remains preserved.  
For example, removing $10{,}000$ samples from MNIST ($N=50{,}000$) reduces the compressed size by $1/5$, while the data distribution remains approximately the same.

\subsection{Diffusion Model Training Procedure}

The training of our diffusion model follows the framework described in the theoretical background and incorporates several practical considerations to ensure consistency across experiments.

\textbf{Controlling Experimental Variables}:  
To isolate the effect of dataset complexity and image resolution, we control all other training variables to prevent confounding factors.

\textbf{Number of Iterations}:  
All models are trained for a fixed total of 5 million iterations.  
This number was chosen empirically to ensure smooth convergence across all datasets.

\textbf{Architecture}:  
We use a conditional U-Net with three levels of encoder–decoder blocks, symmetric skip connections, and self-attention layers at multiple resolutions to capture both local and global dependencies.  
The base channel width is 64, doubling at each successive level.  
Architecture depth and channel width are held constant across experiments to isolate dataset and resolution effects.

\textbf{Optimization Strategy}:  
Training uses the AdamW optimizer with a one-cycle learning rate scheduler and weight decay to facilitate stable convergence.  
A linear noise schedule is applied, and an exponential moving average (EMA) of the model weights is maintained, as EMA weights generally yield higher sample quality at inference.

\textbf{Model Selection}:  
After each epoch, the model is validated on $D_{\text{val}}^{\text{real}}$ using the same MSE objective employed during training.  
This validation allows for consistent monitoring of convergence and ensures comparability across runs.

By carefully controlling these factors, any observed differences in generative performance can be confidently attributed to variations in dataset complexity or image resolution, rather than inconsistencies in architecture, optimization, or training procedure.

\subsection{AI Image Generation}

During image generation, noise is iteratively transformed according to a conditional label.  
Generated datasets are sampled to match the size of the corresponding real subsets:
\[
|D_{\text{train}}^{\text{real}}| = |D_{\text{train}}^{\text{gen}}|, \quad
|D_{\text{val}}^{\text{real}}| = |D_{\text{val}}^{\text{gen}}|, \quad
|D_{\text{test}}^{\text{real}}| = |D_{\text{test}}^{\text{gen}}|.
\]

Figures~\ref{fig:fmnist_cfg} and~\ref{fig:cifar_cfg} illustrate the effect of varying the classifier-free guidance (CFG) parameter on image saturation for FashionMNIST and CIFAR-10.  
Increasing CFG values (0, 2, 10) produces visibly more saturated images, demonstrating how generation parameters can influence dataset characteristics.  

\begin{figure}[htbp]
    \centering
    \includegraphics[width=0.6\textwidth]{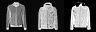}
    \caption{FashionMNIST images with increasing CFG values (0, 2, 10). Saturation increases with higher CFG.}
    \label{fig:fmnist_cfg}
\end{figure}

\begin{figure}[htbp]
    \centering
    \includegraphics[width=0.6\textwidth]{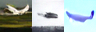}
    \caption{CIFAR-10 images with increasing CFG values (0, 2, 10). Saturation increases with higher CFG.}
    \label{fig:cifar_cfg}
\end{figure}

We do not apply techniques such as Classifier-Free Guidance (CFG) to improve sample quality.  
CFG introduces a weight parameter $w$ that requires careful tuning, which varies across datasets.  
Optimizing $w$ would introduce an uncontrolled variable that could influence discriminator performance.  
Prior work shows that CFG affects image saturation \citep{sadat2025eliminatingoversaturationartifactshigh}, potentially providing a trivial signal for discriminators.  

Other methods, such as Autoguidance \citep{karras2024guidingdiffusionmodelbad} or APG \citep{sadat2025eliminatingoversaturationartifactshigh}, mitigate this issue or improve FID.  
However, these methods also require dataset-specific optimization, which would similarly introduce non-controlled variables.  
By avoiding these techniques, we ensure that discriminator evaluation reflects intrinsic dataset and model characteristics rather than artifacts of sampling parameter tuning.

Table~\ref{tab:fid_results} reports the Fréchet Inception Distance (FID) between real and generated datasets, alongside the FID between training and validation subsets of the real data.
The Real-AI FID provides a quantitative measure of how closely the diffusion model replicates the distribution of real images, with lower values indicating higher fidelity. The Train-Val FID serves as a baseline and a lower bound on achievable FID, capturing the natural variability within the real dataset itself.
Across datasets, FID values vary significantly, reflecting differences in dataset complexity, image diversity, and the inherent difficulty of generation.
For simpler datasets like FashionMNIST or KMNIST, Real-AI FID is low and close to the Train-Val baseline, while complex datasets such as ChestMNIST or PathMNIST show substantially higher FID, indicating that the model struggles more to capture intricate visual patterns.
Interestingly, datasets with intermediate complexity such as SVHN and EuroSAT are the ones with highest FID scores.
These results highlight both the strengths and limitations of the diffusion model in reproducing diverse datasets and provide context for subsequent discriminator evaluations.

\begin{table}[ht]
\centering
\caption{Fréchet Inception Distance (FID) comparisons for real versus generated datasets (Real-AI FID) and between training and validation splits of real data (Train-Val FID) across multiple datasets. Lower FID indicates higher fidelity to the real distribution.}
\begin{tabular}{lcc}
\toprule
\textbf{Dataset} & \textbf{Real-AI FID} & \textbf{Train-Val FID} \\
\midrule
MNIST\cite{mnist} & 14.97 & 0.74 \\
KMNIST\cite{torchvisiondatasets} & 8.19  & 0.95 \\
FashionMNIST\cite{fashionmnist} & 6.27  & 1.47 \\
EMNIST\cite{emnist} & 11.01 & 0.69 \\
SVHN\cite{svhn} & 68.71 & 1.90 \\
EuroSAT\cite{eurosat} & 44.20 & 5.23 \\
BloodMNIST\cite{medmnistv2,medmnistv1} & 9.35  & 1.99 \\
GTSRB\cite{stallkamp2011} & 15.64 & 3.04 \\
Country211\cite{torchvisiondatasets} & 23.27 & 4.39 \\
CIFAR-100\cite{cifar10100} & 21.91 & 3.69 \\
CIFAR-10\cite{cifar10100} & 17.45 & 3.23 \\
FER-2013\cite{fer2013} & 15.30 & 3.32 \\
OrganAMNIST\cite{medmnistv2,medmnistv1} & 14.11 & 2.82 \\
Food101\cite{food101} & 19.41 & 2.82 \\
DermaMNIST\cite{medmnistv2,medmnistv1} & 19.99 & 6.22 \\
OCTMNIST\cite{medmnistv2,medmnistv1} & 10.26 & 1.04 \\
ChestMNIST\cite{medmnistv2,medmnistv1} & 15.45 & 1.16 \\
PathMNIST\cite{medmnistv2,medmnistv1} & 14.77 & 1.51 \\
Anime Face Dataset \cite{spencerchurchillbrianchao2019} & 13.64 & 2.36 \\
\bottomrule
\end{tabular}
\label{tab:fid_results}
\end{table}

\subsection{Discriminator Model Training}\label{appendix: discriminator training}

This section describes the training procedure for discriminators tasked with distinguishing real images from AI-generated ones.  \Cref{tab:discriminator_overview} gives an overview of the models used.

The input data consists of two sources: real images and generated images.  
To ensure a balanced dataset, the number of real samples is matched to the number of AI-generated samples.  

The split of $D^{\text{real}}$ corresponds to that used during diffusion model training, which only used $D_{\text{train}}^{\text{real}}$ and $D_{\text{val}}^{\text{real}}$.  

For discriminator training, the number of AI-generated images in $D_{\text{train}}^{\text{gen}}$ is capped at 50,000, yielding a maximum total of 100,000 images.  
The training, validation, and test sets for the discriminator contain real and generated images in equal proportion, forming $D^{\text{discr}} = D^{\text{real}} \cup D^{\text{gen}}$.  

Each image is assigned a binary label:
\[
y = 
\begin{cases}
1 & \text{if the image is real}, \\
0 & \text{if the image is AI-generated}.
\end{cases}
\]

We evaluate six discriminator variants, differing in architecture and input representation, while keeping hyperparameters consistent.  

All models are trained using the AdamW optimizer with 
\[
\alpha = 2 \times 10^{-4}, \quad \beta_1 = 0.5, \quad \beta_2 = 0.999.
\]

The loss function is Binary Cross-Entropy with logits.  
Training is performed for 1 million iterations.  

After each epoch, the discriminator is evaluated on the validation set $D_{\text{val}}^{\text{discr}}$.  
The model achieving the lowest validation loss is selected for final evaluation.  

\begin{table}[t]
\centering
\caption{Overview of discriminator variants, including input modality and the number of tunable parameters.}
\label{tab:discriminator_overview}
\begin{tabular}{ccc}
\toprule
\textbf{Model} & \textbf{Input Modality} & \textbf{\# Tunable Parameters} \\
\midrule
Base & Normalized & $\approx 40{,}000$ \\
Base & Fourier & $\approx 40{,}000$ \\
Big & Normalized & $\approx 520{,}000$ \\
Big & Fourier & $\approx 520{,}000$ \\
ResNet (Frozen) & Normalized & $\approx 500$ \\
ResNet (Fine-tuned) & Normalized & $\approx 11{,}000{,}000$ \\
\bottomrule
\end{tabular}
\end{table}

\section{Computational Resources}\label{appendix: compute used}
All experiments were conducted on an internal compute cluster equipped with RTX 3090 GPUs.
In total, we logged 2,261 GPU-hours for both exploratory experiments and the reported results.

\newpage
\section*{NeurIPS Paper Checklist}

\begin{enumerate}

\item {\bf Claims}
    \item[] Question: Do the main claims made in the abstract and introduction accurately reflect the paper's contributions and scope?
    \item[] Answer: \answerYes{} 
    \item[] Justification: 

\item {\bf Limitations}
    \item[] Question: Does the paper discuss the limitations of the work performed by the authors?
    \item[] Answer: \answerYes{} 
    \item[] Justification: See \Cref{sec: limitations}.
    \item[] Guidelines:
    \begin{itemize}
        \item The answer NA means that the paper has no limitation while the answer No means that the paper has limitations, but those are not discussed in the paper. 
        \item The authors are encouraged to create a separate "Limitations" section in their paper.
        \item The paper should point out any strong assumptions and how robust the results are to violations of these assumptions (e.g., independence assumptions, noiseless settings, model well-specification, asymptotic approximations only holding locally). The authors should reflect on how these assumptions might be violated in practice and what the implications would be.
        \item The authors should reflect on the scope of the claims made, e.g., if the approach is only tested on a few datasets or with a few runs. In general, empirical results often depend on implicit assumptions, which should be articulated.
        \item The authors should reflect on the factors that influence the performance of the approach. For example, a facial recognition algorithm may perform poorly when image resolution is low or images are taken in low lighting. Or a speech-to-text system might not be used reliably to provide closed captions for online lectures because it fails to handle technical jargon.
        \item The authors should discuss the computational efficiency of the proposed algorithms and how they scale with dataset size.
        \item If applicable, the authors should discuss possible limitations of their approach to address problems of privacy and fairness.
        \item While the authors might fear that complete honesty about limitations might be used by reviewers as grounds for rejection, a worse outcome might be that reviewers discover limitations that aren't acknowledged in the paper. The authors should use their best judgment and recognize that individual actions in favor of transparency play an important role in developing norms that preserve the integrity of the community. Reviewers will be specifically instructed to not penalize honesty concerning limitations.
    \end{itemize}

\item {\bf Theory assumptions and proofs}
    \item[] Question: For each theoretical result, does the paper provide the full set of assumptions and a complete (and correct) proof?
    \item[] Answer: \answerYes{} 
    \item[] Justification: 

    \item {\bf Experimental result reproducibility}
    \item[] Question: Does the paper fully disclose all the information needed to reproduce the main experimental results of the paper to the extent that it affects the main claims and/or conclusions of the paper (regardless of whether the code and data are provided or not)?
    \item[] Answer: \answerYes{} 
    \item[] Justification: See for instance \Cref{appendix: discriminator training}.

\item {\bf Open access to data and code}
    \item[] Question: Does the paper provide open access to the data and code, with sufficient instructions to faithfully reproduce the main experimental results, as described in supplemental material?
    \item[] Answer: \answerYes{} 
    \item[] Justification: The code is provided in a zip file

\item {\bf Experimental setting/details}
    \item[] Question: Does the paper specify all the training and test details (e.g., data splits, hyperparameters, how they are chosen, type of optimizer, etc.) necessary to understand the results?
    \item[] Answer: \answerYes{} 
    \item[] Justification: 

\item {\bf Experiment statistical significance}
    \item[] Question: Does the paper report error bars suitably and correctly defined or other appropriate information about the statistical significance of the experiments?
    \item[] Answer: \answerYes{} 
    \item[] Justification: See \Cref{tab:accuracy_comparison}.
    \item[] Guidelines:
    \begin{itemize}
        \item The answer NA means that the paper does not include experiments.
        \item The authors should answer "Yes" if the results are accompanied by error bars, confidence intervals, or statistical significance tests, at least for the experiments that support the main claims of the paper.
        \item The factors of variability that the error bars are capturing should be clearly stated (for example, train/test split, initialization, random drawing of some parameter, or overall run with given experimental conditions).
        \item The method for calculating the error bars should be explained (closed form formula, call to a library function, bootstrap, etc.)
        \item The assumptions made should be given (e.g., Normally distributed errors).
        \item It should be clear whether the error bar is the standard deviation or the standard error of the mean.
        \item It is OK to report 1-sigma error bars, but one should state it. The authors should preferably report a 2-sigma error bar than state that they have a 96\% CI, if the hypothesis of Normality of errors is not verified.
        \item For asymmetric distributions, the authors should be careful not to show in tables or figures symmetric error bars that would yield results that are out of range (e.g. negative error rates).
        \item If error bars are reported in tables or plots, The authors should explain in the text how they are calculated and reference the corresponding figures or tables in the text.
    \end{itemize}

\item {\bf Experiments compute resources}
    \item[] Question: For each experiment, does the paper provide sufficient information on the computer resources (type of compute workers, memory, time of execution) needed to reproduce the experiments?
    \item[] Answer: \answerYes{} 
    \item[] Justification: See overall numbers in \Cref{appendix: compute used}.
    
\item {\bf Code of ethics}
    \item[] Question: Does the research conducted in the paper conform, in every respect, with the NeurIPS Code of Ethics \url{https://neurips.cc/public/EthicsGuidelines}?
    \item[] Answer: \answerYes{} 
    \item[] Justification: 

\item {\bf Broader impacts}
    \item[] Question: Does the paper discuss both potential positive societal impacts and negative societal impacts of the work performed?
    \item[] Answer: \answerNA{} 
    \item[] Justification: There are no direct impacts from this work. The model sizes are too small for general usage of deep fakes.
    
\item {\bf Safeguards}
    \item[] Question: Does the paper describe safeguards that have been put in place for responsible release of data or models that have a high risk for misuse (e.g., pretrained language models, image generators, or scraped datasets)?
    \item[] Answer: \answerNA{} 
    \item[] Justification: 

\item {\bf Licenses for existing assets}
    \item[] Question: Are the creators or original owners of assets (e.g., code, data, models), used in the paper, properly credited and are the license and terms of use explicitly mentioned and properly respected?
    \item[] Answer: \answerYes{} 
    \item[] Justification: 

\item {\bf New assets}
    \item[] Question: Are new assets introduced in the paper well documented and is the documentation provided alongside the assets?
    \item[] Answer: \answerNA{} 
    \item[] Justification: 
    \item[] Guidelines:
    \begin{itemize}
        \item The answer NA means that the paper does not release new assets.
        \item Researchers should communicate the details of the dataset/code/model as part of their submissions via structured templates. This includes details about training, license, limitations, etc. 
        \item The paper should discuss whether and how consent is obtained from people whose asset is used.
        \item At submission time, remember to anonymize your assets (if applicable). You can either create an anonymized URL or include an anonymized zip file.
    \end{itemize}

\item {\bf Crowdsourcing and research with human subjects}
    \item[] Question: For crowdsourcing experiments and research with human subjects, does the paper include the full text of instructions given to participants and screenshots, if applicable, as well as details about compensation (if any)? 
    \item[] Answer: \answerNA{} 
    \item[] Justification: 

\item {\bf Institutional review board (IRB) approvals or equivalent for research with human subjects}
    \item[] Question: Does the paper describe potential risks incurred by study participants, whether such risks are disclosed to the subjects, and whether Institutional Review Board (IRB) approvals (or an equivalent approval/review based on the requirements of your country or institution) are obtained?
    \item[] Answer: \answerNA{} 
    \item[] Justification: 

\item {\bf Declaration of LLM usage}
    \item[] Question: Does the paper describe the usage of LLMs if it is an important, original, or non-standard component of the core methods in this research? Note that if the LLM is used only for writing, editing, or formatting purposes and does not impact the core methodology, scientific rigorousness, or originality of the research, declaration is not required.
    \item[] Answer: \answerNA{} 
    \item[] Justification: 

\end{enumerate}

\end{document}